\documentclass[letterpaper, 10 pt, journal, twoside]{IEEEtran}
\usepackage{cite}
\usepackage{amsmath,amssymb,amsfonts}
\usepackage{algorithmic}
\usepackage{graphicx}
\usepackage{amsthm}
\usepackage{breqn}
\usepackage{textcomp}
\usepackage{color}
\usepackage{xcolor}
\usepackage{hyperref}
\usepackage{xcolor}
\usepackage{algorithm}
\usepackage{placeins}
\usepackage{float}
\usepackage{xcolor}
\usepackage{booktabs}
\usepackage{breqn}
\usepackage{subcaption}
\usepackage{algorithm}
\usepackage{amsmath}
\usepackage{pdflscape}
\usepackage{hyperref}
\usepackage{titlesec}
\titlespacing{\subsection}{0pt}{*0}{*0}
\titlespacing{\section}{0pt}{*0}{*0}
\usepackage{longtable}
\usepackage{lipsum} 

\newtheorem{theorem}{\bf Theorem}[section]

\newtheorem{remark}[theorem]{\bf Remark}
\newtheorem{lemma}[theorem]{\bf Lemma}
\usepackage{algorithm}
\usepackage{multirow}
\usepackage{graphicx}
\usepackage{float}
\usepackage{placeins}

\newcommand{\beano}{\begin{eqnarray*}}
	\newcommand{\eeano}{\end{eqnarray*}}

\makeatletter

\def\BibTeX{{\rm B\kern-.05em{\sc i\kern-.025em b}\kern-.08em
    T\kern-.1667em\lower.7ex\hbox{E}\kern-.125emX}}
\begin{document}
\title{A Multi-Step Minimax Q-Learning Algorithm for Two-Player Zero-Sum Markov Games
}

\author{Shreyas S R\IEEEauthorrefmark {1}, \, Antony Vijesh\IEEEauthorrefmark {2}
}

\maketitle

\begin{abstract} An interesting iterative procedure is proposed to solve two-player zero-sum Markov games. Under suitable assumption, the boundedness of the proposed iterates is obtained theoretically. Using results from stochastic approximation, the almost sure convergence of the proposed multi-step minimax Q-learning is obtained theoretically. More specifically, the proposed algorithm converges to the game theoretic optimal value with probability one, when the model information is not known. Numerical simulation authenticates that the proposed algorithm is effective and easy to implement.
\end{abstract}

\footnotetext[1]{Department of Mathematics,
	Indian Institute of Technology (IIT) Indore, India, 452020. \texttt{Email:shreyassr123@gmail.com, phd1901241006@iiti.ac.in}, \textcolor{black}{Shreyas S R  is grateful to CSIR, India (File No. 09/1022(0088)2019-EMR-I) for the financial support}}
\footnotetext[2]{
	Department of Mathematics, IIT Indore, Simrol, Indore-452020, Madhya Pradesh, India. \texttt{Email:vijesh@iiti.ac.in}}

	\noindent
	\section{ Introduction}
Single-agent reinforcement learning (RL) deals with the task of finding an optimal policy or strategy in a stochastic environment. Markov Decision Process (MDPs) provides the underlying framework for solving the problem of finding an optimal strategy in these decision-making problems with one agent \cite{MR3889951}. Single-agent RL has been successfully implemented in robot navigation, agent-based production scheduling, economics, and traffic signal control. However, many practical environments involve more than one agent competing against each other to optimize their strategies. This type of problem can be studied systematically using the concept of Markov games from game theory. Markov games can be considered as a generalization of MDPs \cite{MR1418636}. Similar to the problem of finding an optimal strategy in MDPs, the problem of finding an optimal strategy in Markov games is an important problem. Using successive iterative scheme, when the model information is completely available, Shapley \cite{shapley1953stochastic} obtained the optimal strategy for two-player zero-sum Markov games (TZMG). By modifying the Q-learning algorithm for MDP, M. L. Littman \cite{littman1994markov} extends the Q-learning algorithm to solve TZMG when the information of the model is not completely available. This algorithm is also called as minimax Q-learning. Under suitable assumptions, M. L. Littman and C. Szepesvari \cite{littman1996generalized} show that this algorithm converges to the game theoretic optimal value with probability one. Consequently, Markov games become an elegant framework for reinforcement learning with more than one agent, though efforts were made without Markov games \cite{tan1993multi,yanco1993adaptive}.  \\
\hspace*{0.5cm}The Q-learning algorithm for TZMG \cite{littman1994markov} is further designed to handle both coordination games \cite{claus1998dynamics}, and general sum games \cite{hu1998multiagent}. The Q-learning algorithm for general sum stochastic game  \cite{hu1998multiagent} was refined by M. Bowling \cite{bowling2000convergence} and further generalized by J. Hu and M. P. Wellman \cite{hu2003nash}. This algorithm in \cite{hu2003nash} is called as Nash Q-learning. By introducing the attribution friend and foe, M. L. Littman \cite{littman2001friend} relaxed the restrictive assumption in the Q-learning algorithm for general sum stochastic games \cite{hu1998multiagent}. By introducing WoLF principle, M. Bowling and M. Veloso \cite{bowling2001rational,bowling2002multiagent} tried to develop multi-agent reinforcement learning satisfying rational property, using the stochastic game framework. Minimax Q-learning algorithm, together with value function approximation, was studied by M. G. Lagoudakis and R. Parr \cite{lagoudakis2002value} for TZMG. Further, A. Greenwald and K. Hall \cite{greenwald2003correlated} developed an algorithm called correlated Q-learning for solving general sum Markov games. This algorithm can be considered as a generalization of Littman's \cite{littman2001friend} friend-or-foe Q-learning and Hu and Wellman's \cite{hu2003nash} Nash Q-learning. In this direction, various other reinforcement algorithms are also suitably modified to study TZMG. B. Banerjee et al. \cite{banerjee2001fast} extended the SARSA algorithm as minimax-SARSA algorithm to find Nash equilibrium in TZMG. Recently, the SOR Q-learning \cite{kamanchi2019successive} for MDP is suitably updated to handle TZMG in \cite{diddigi2022generalized}. Similarly, the switching system technique for MDP \cite{lee} was modified to study minimax Q-learning for TZMG \cite{lee2023finite}. Various developments in the aspects of multi-agent reinforcement learning algorithms have been comprehensively reported in \cite{busoniu2008comprehensive}\cite{zhang2021multi}. \\
Though various single-step RL algorithms with single-agent suitably updated for multi-agent RL algorithms in the literature, very few results are available for multi-step multi-agent RL algorithms. In 2001, B. Banerjee et al. \cite{banerjee2001fast} observed that Peng's and William's  Q($\lambda$) \cite{peng1994incremental} version of minimax algorithm performs better than the single-step SARSA minimax \cite{banerjee2001fast} and Littman's minimax \cite{littman1994markov} algorithms. Unfortunately, they were not able to obtain the theoretical convergence of this multi-step multi-agent RL algorithm. It is worth mentioning that the theoretical convergence of Peng's and William's  Q($\lambda$) algorithm for single-agent was also not available during this period. Recently, in 2021, Kozuno et al. \cite{kozuno2021revisiting}, proved the convergence of Peng's and William's  Q($\lambda$) algorithm for a single-agent. Studying convergence of multi-step algorithms for MDP are available in \cite{kozuno2021revisiting,harutyunyan2016q,munos2016safe}. The practicality of employing multi-step algorithms has long been a subject of inquiry. This question finds a comprehensive response in \cite{cichosz1994truncating} and \cite{van2016true}.  While numerous multi-step algorithms have been developed within the framework of single-agent RL, there is still much work to be done in exploring their application and advantages within the context of MARL. However, the recent success of algorithms like AlphaGo, which uses a multi-step Monte Carlo Tree Search (MCTS) algorithm, is an important step in that direction \cite{silver2016mastering}. In a similar vein, the proposed algorithm also explores the notion of multi-step RL algorithms in the context of TZMG. \\
In \cite{tsql}, two interesting two-step algorithms were proposed to handle MDP and obtained their theoretical convergence results. Numerical results in \cite{tsql} showed that the two-step method has good flexibility and performs better than the single-step methods in some benchmark problems. In this manuscript, we customize this idea from \cite{tsql} to handle Markov games. More specifically, a multi-step minimax Q-learning (MMQL) algorithm is proposed to obtain the Nash equilibrium in the TZMG. The contributions of this manuscript are as follows:
\begin{itemize}
	\item The boundedness of the proposed MMQL algorithm is provided.  
	\item Using the techniques available in the stochastic approximation theory, the convergence of the MMQL algorithm is presented.
	\item The proposed algorithm's performance is empirically compared with recent algorithms, demonstrating the superiority of the proposed algorithm both in terms of error and time.	\end{itemize}	
	
The structure of the manuscript is as follows: Section 2 covers preliminaries, notations, and essential results. Section 3 discusses the proposed algorithm in detail. Section 4 presents the main findings and demonstrates the convergence of the proposed algorithm. Section 5 includes numerical experiments. Finally, Section 6 draws conclusions.
\section{Preliminaries and Background}
\noindent
In this manuscript, we consider two-player zero-sum Markov games. Formally, described by a six-tuple $(S,A,B,p,r,\alpha)$, where $S$ is a set of states, $A$ and $B$ are the set of actions that can be performed by agent 1, and agent 2 respectively. The system evolves according to the transition probability $p$. More specifically, $p(j|i,a,b)$ represents the probability of transitioning to state $j$ when agent $1$, and agent $2$ select actions $a$, and $b$, respectively, while in state $i$. The real number $r(i,a,b)$ denotes the reward or payoff obtained by player 1 when the actions $a$, and $b$ are taken by agent 1, and agent 2 respectively, while in state $i$. In zero-sum Markov games, it is important to note that the payoff for agent 2 is negative of the payoff acquired by agent 1. Further, $\alpha \in [0,1)$ denotes the discount factor. It is important to mention that this manuscript deals with finite state space $S$, and the finite action spaces $A$, and $B$.  Also, the norm used is the max-norm, i.e., for any $z \in \mathbb{R}^d$, $\|z\| := \max_{1\leq i \leq d}|z(i)|$. In addition, let $R_{\max}$ denote $\max_{i,a,b} |r(i,a,b)|$. 
\\
The objective of the agents in TZMG is to learn an optimal strategies $\mu^*_1: S \rightarrow \Delta^{|A|}$, and $\mu^*_2: S \rightarrow \Delta^{|B|}$, where $\Delta^{|X|}$ denote the probability simplex in $\mathbb{R}^{|X|}$. The existence of such policies for the discounted finite state action Markov game  is guaranteed in the classical game theory \cite{shapley1953stochastic}. The problem of finding an optimal strategy $\mu^*_1$, and $\mu^*_2$ reduces to solving the following system of non-linear equations analogous to Bellman's optimality equations from dynamic programming given by
\begin{align}\label{eq1}
Q^*(i,a,b)=&r(i,a,b)+\alpha\sum_{j=1}^{|S|}p(j|i,a,b)\, val[Q^*(j)],\\ &\quad \quad \hspace*{2cm} \forall (i,a,b)\in S\times A\times B \nonumber,
\end{align}
where the $val$ operator is defined for any $|A| \times |B|$ matrix $M$ as follows:
\begin{equation}
val[M] := \min_{y \in \Delta^{|B|}} \max_{x \in \Delta^{|A|}} x^tMy, \quad \text{for any}\, M.
\end{equation}
Similar to the Q-Bellman operator in MDPs, the minimax Q-Bellman operator follows\\
$H:\mathbb{R}^{|S\times A \times B|}\rightarrow \mathbb{R}^{|S\times A \times B|}$ is defined as
\begin{equation}
(HQ)(i,a,b)=r(i,a,b)+\alpha \sum_{j=1}^{|S|}p(j|i,a,b)val[Q(j)].
\end{equation}
The subsequent analysis relies on the properties of the $val$ operator, which are presented in the following lemma. 
\begin{lemma}[Proposition G.4, in \cite{MR1418636}]\label{l1}
	Let $M = [m_{ij}]$, and $N = [n_{ij}]$ be two $|A| \times |B|$ matrices. Then the following holds true :
	
	(i) $| val[M] - val[N] |  \,\leq \max_{i,j} |m_{ij} - n_{ij} | = \| M - N\|.$\\
	
	(ii) $| val[M] | \, \leq \|M\|.$
\end{lemma}
\noindent
Using Lemma \ref{l1}, one can conclude that $H$ is a contraction mapping with respect to max-norm, and $\alpha$ is the contraction factor. Consequently, $H$ has a unique fixed point, say $Q^*$. From the optimal $Q^*$, an optimal policies $\mu_1^*$, and $\mu_2^*$ obtained as follows:
$ (\mu_1^*(i),\mu_2^*(i))\in \arg val[Q^*(i)], \quad \forall i \in S.$\\
M. L. Littman and C. Szepesvari \cite{littman1996generalized} proved that the following iterative scheme:
\begin{align*}
	&Q_{n+1}(i,a,b) = (1-\beta_n(i,a,b)) Q_n(i,a,b)\\&\quad \quad \quad +\beta_n(i,a,b)   (r(i,a,b)+\alpha \, val[ Q_n(j)]),
\end{align*}
converges to the optimal Q-value if all combinations of states and actions are chosen infinitely often and the learning rate satisfies Robbins and Monro condition. In other words, minimax Q-learning converges to the fixed point of minimax Q-Bellman operator. \\
The following lemma from stochastic approximation due to S. Singh et al. \cite{singh2000convergence} is used to prove the almost convergence of the proposed multi-step minimax Q-learning algorithm.
\begin{lemma}(Lemma 1, \cite{singh2000convergence})\label{fl}
	Let $X$ be a finite set. Let $ \left(\Psi_{n},F_n,\beta_n\right) $ be a stochastic process such that $\Psi_{n},F_n,\beta_n:X\rightarrow \mathbb{R}$ and satisfies the recurrence relation :
	\begin{equation*}
	\Psi_{n+1}(y)=(1-\beta_n(y))\Psi_{n}(y)+\beta_n(y)F_n(y), \; where \hspace{0.1cm} y \in X.\end{equation*}
	Let $\mathcal{F}_n$, $n=0,1,2,...$ be an increasing sequence of sigma-fields that includes history of the process such that $\beta_0$, and $\Psi_0$ are $\mathcal{F}_0$ measurable, and for $n\geq 1$, $\beta_n$, $\Psi_n$, and $F_{n-1}$ are $\mathcal{F}_n$ measurable. Then $\Psi_{n}$ $\rightarrow$ 0 with probability one $(w.p.1)$ as $n$ $\rightarrow$ $\infty$, if the following condition hold:\\
	1. $0\leq\beta_n(y)\leq 1$, $\sum_{n=1}^{\infty}\beta_n(y)=\infty$, $\sum_{n=1}^{\infty}\beta^2_n(y)<\infty$ w.p.1.\\
	2. $\|E[F_n|\mathcal{F}_n]\| \leq \kappa \|\Psi_n\|+\zeta_n$, where $\kappa \in[0,1)$, and $\zeta_n$ $\rightarrow$ 0 $w.p.1$, and $\|.\|$ is any weighted max-norm.\\
	3. $Var[ F_n(y)|\mathcal{F}_n]\leq K(1+ \|\Psi_n\|)^2$, where $K$ is some constant. 
	
\end{lemma}
\noindent
Boundedness plays an important role in the convergence of many iterative schemes in RL. Usually by assuming the boundedness of the Q-iterates \cite{diddigi2022generalized} or by proving the boundedness of the Q-iterates \cite{singh2000convergence}\cite{borkar2000ode} the theoretical convergence analysis will be performed. In this manuscript, the boundedness of the Q-iterates proved theoretically under suitable assumption on rewards using the following lemma.
\begin{lemma}(Proposition 3.1, \cite{MR1976398})\label{k1}
	If $\sum_{n=1}^{\infty}|a_n| < \infty$, then the product $\prod_{n=1}^{\infty}(1+a_n)$ converges.
\end{lemma}

With the above results at hand, we present the proposed algorithm in the next section and its convergence behaviour in Section 4.
\section{Proposed Algorithm}
\noindent

The proposed algorithm is presented in this section. As discussed in the previous section, for standard minimax Q-learning, at every step of the iteration, we receive a reward $r'$ and the next state $s'$ at state $s$. In the proposed algorithm, however, at every step of the iteration, at each step of the iteration, in addition to $r'$ and $s'$, depending on the value of multi-step parameter $k$, we generate the following sample at the $n^{th}$ time step:
\begin{equation*}
\begin{split}
\{s_n,a_n,b_n,r^{1}_n,s^{1}_n,a^{1}_n,b^{1}_n,r^{2}_n,s^{2}_n,a^{2}_n,b^{2}_n,r^{3}_n,s^{3}_n,a^{3}_n, b^{3}_n,\\
r^{4}_n,\cdots, s^{k-1}_n,a^{k-1}_n,b^{k-1}_n,r_{n+1},s_{n+1},a_{n+1},b_{n+1}\}.
\end{split}
\end{equation*}
Using this information, the update rule for the multi-step algorithm is as follows:
\begin{align} \label{main}
&Q_{n+1}(s_n,a_n,b_n)\notag\\ &= (1-\beta_n(s_n,a_n,b_n))Q_n(s_n,a_n,b_n)+ \beta_n(s_n,a_n,b_n) \Big( \notag \\
& \quad r^{1}_n + \alpha\, val[Q_n(s^{1}_n)] + \alpha\theta^{1}_n \left( r^{2}_n + \alpha\, val[Q_n(s^{2}_n)] \right)+\notag \\
& \quad \quad  \alpha^2 \theta^{2}_n \left( r^{3}_n + \alpha\, val[Q_n(s^{3}_n)] \right)+\cdots \notag \\
&\quad \quad   \cdots + \alpha^{k-1} \theta^{k-1}_n \left( r_{n+1} + \alpha\, val[Q_n(s_{n+1})] \right) \Big).
\end{align}

\begin{algorithm*}[!t]
	\caption{Multi-step minimax Q-learning algorithm (MMQL)}
	\begin{algorithmic}[1]
		\REQUIRE Initial Q-Vector: $Q_0$; Discount factor: $\alpha$; Step size rule $\beta_n$; Multi-step parameter $k$, Sequence $\theta^{i}_n$, $i=1,2,...,k-1$; Number of iterations: $T$; A set $\{s_n,a_n,b_n,r^{1}_n,s^{1}_n,a^{1}_n,b^{1}_n,\dots,s^{k-1}_n,a^{k-1}_n,b^{k-1}_n,r_{n+1}, s_{n+1},a_{n+1}\}_{n=0}^{\infty}$, where every state-action pair appears infinitely often.
		\FOR {$n = 0, 1, \ldots, T-1$}
		\STATE For actions $a_n$, $b_n$ at $s_n$, observe the set
		\STATE Observe $s^{1}_n$ and $r^{1}_n$
		\STATE For $a^{1}_n$ at $s^{1}_n$, observe $s^{2}_n$ and $r^{2}_n$
		\STATE Continue the process until $k-1$ steps, i.e., at $s^{k-1}_n$, $a^{k-1}_n$, observe $r_{n+1}$, $s_{n+1}$
		\STATE \textbf{Update rule:}
		\begin{align*}
		&Q_{n+1}(s_n,a_n,b_n) \\&= (1-\beta_n(s_n,a_n,b_n))Q_n(s_n,a_n,b_n) + \beta_n(s_n,a_n,b_n) \Big( r^{1}_n + \alpha\, val[Q_n(s^{1}_n)]  + \alpha\theta^{1}_n \left( r^{2}_n + \alpha\, val[Q_n(s^{2}_n)] \right) + \\& \quad \quad \quad \quad \alpha^2 \theta^{2}_n \left( r^{3}_n + \alpha\, val[Q_n(s^{3}_n)] \right) + \cdots + \alpha^{k-1} \theta^{k-1}_n \left( r_{n+1} + \alpha\, val[Q_n(s_{n+1})] \right) \Big).
		\end{align*}
		\STATE Set $s_n = s_{n+1}$
		\ENDFOR
		\RETURN $Q_{T-1}$
	\end{algorithmic}
\end{algorithm*}
	The following remark provides conditions on \( \theta^{i}_n \), which will be clarified in the next section. Specifically, they will be necessary for proving the boundedness of the proposed algorithm.
\begin{remark}\label{rem1}
	For \( i = 1, 2, \dots, k-1 \), let \( \theta^{i}_n \) be a sequence of real numbers satisfying the following conditions:
	\begin{enumerate}
		\item \( |\theta^{i}_n| \leq 1, \quad \forall n \in \mathbb{N}. \)
		\item \( |\theta^{i}_n| \) is monotonically decreasing to zero.
		\item \( \sum_{n=1}^{\infty} \beta_n |\theta^{i}_n| < \infty. \)
	\end{enumerate}
\end{remark}

\section{Main Results and Convergence Analysis}
\noindent
The following section discusses the almost sure convergence of the proposed algorithm using Lemma \ref{fl}, discussed in Section 2. Before presenting the main convergence theorem, the boundedness of the multi-step minimax Q-learning will be discussed in the following result.
\begin{lemma} Let $Q_n(s,a,b)$ be the value corresponding to a state $s$ and action pair $(a,b)$ at $n^{th}$ iteration of MMQL with multi-step parameter $k$. If $\|Q_0\| \leq \frac{R_{\max}}{1-\alpha}$, then $\|Q_n\| \leq M,\, \forall n \in \mathbb{N}, \quad  \text{where}\; M= \frac{R_{\max}}{1-\alpha}(1+\alpha |\theta^{1}_0|+\alpha^2 |\theta^{2}_0|+...+\alpha^{k-1} |\theta^{k-1}_0| ) \prod_{i=1}^{\infty} (1+\beta_i|\theta^{1}_i| \alpha^2+\beta_i|\theta^{2}_i| \alpha^3+...++\beta_i|\theta^{k-1}_i| \alpha^{k})$.
\end{lemma}
\begin{proof}
	We obtain this result using the induction principle. Let $(s,a,b) \in S \times A \times B$,
	\begin{align*} 
	&|Q_{1}(s,a,b)|\\ &= \bigg|(1-\beta_0)Q_0(s,a,b)+ \beta_0 \Big( r^{1} + \alpha\, val[Q_0(s^{1})]  \\
	&\quad + \alpha\theta^{1} \left( r^{2} + \alpha\, val[Q_0(s^{2})] \right)+ \alpha^2 \theta^{2} \left( r^{3} + \alpha\, val[Q_0(s^{3})] \right) \notag \\
	&\quad \quad  + \cdots + \alpha^{k-1} \theta^{k-1} \left( r_{1} + \alpha\, val[Q_0(s_{1})] \right) \Big)\bigg|\\
	&\leq (1-\beta_0)\|Q_{0}\| + \beta_0 \left(R_{\max} + \alpha \|Q_{0}\| \right)+ \beta_0 \alpha |\theta^{1}_0| (R_{\max}+\\
	&\quad \alpha \|Q_{0}\|)+ \cdots+\beta_0 \alpha^{k-1} |\theta^{k-1}_0| (R_{\max}+\alpha \|Q_{0}\|)\\
	& \leq (1-\beta_0)\frac{R_{\max}}{1-\alpha}+\beta_0\frac{R_{\max}}{1-\alpha}+\beta_0 \alpha |\theta^{1}_0| \frac{R_{\max}}{1-\alpha}+\cdots\\
	& \cdots+\beta_0 \alpha^{k-1} |\theta^{k-1}_0| \frac{R_{\max}}{1-\alpha}\\
	& \leq \frac{R_{\max}}{1-\alpha} (1+\alpha |\theta^{1}_0|+\alpha^2 |\theta^{2}_0|+\cdots+\alpha^{k-1} |\theta^{k-1}_0|).
	\end{align*}
	Hence, we have $\max_{(s,a,b)} |Q_1(s,a,b)| = \|Q_1\| \leq\frac{R_{\max}}{1-\alpha}(1+\alpha |\theta^{1}_0|+\alpha^2 |\theta^{2}_0|+\cdots+\alpha^{k-1} |\theta^{k-1}_0|)\leq M$. For $l=2,...,n$,  assume that $\|Q_l\| \leq \frac{R_{\max}}{1-\alpha}(1+\alpha |\theta^{1}_0|+\alpha^2 |\theta^{2}_0|+\cdots+\alpha^{k-1} |\theta^{k-1}_0|) \prod_{i=1}^{l-1} (1+\beta_i|\theta^{1}_i| \alpha^2+\beta_i|\theta^{2}_i| \alpha^3+\cdots+\beta_i|\theta^{k-1}_i| \alpha^{k})\leq M$ is true . The proof will be completed if $\|Q_{n+1}\| \leq \frac{R_{\max}}{1-\alpha}(1+\alpha |\theta^{1}_0|+\alpha^2 |\theta^{2}_0|+\cdots+\alpha^{k-1} |\theta^{k-1}_0|) \prod_{i=1}^{n} (1+\beta_i|\theta^{1}_i| \alpha^2+\beta_i|\theta^{2}_i| \alpha^3+\cdots+\beta_i|\theta^{k-1}_i| \alpha^{k})$. Define $L =  \frac{R_{\max}}{1-\alpha}(1+\alpha |\theta^{1}_0|+\alpha^2 |\theta^{2}_0|+\cdots+\alpha^{k-1} |\theta^{k-1}_0|) \prod_{i=1}^{l-1} (1+\beta_i|\theta^{1}_i| \alpha^2+\beta_i|\theta^{2}_i| \alpha^3+\cdots+\beta_i|\theta^{k-1}_i| \alpha^{k})$. Now,
	\begin{align*}
		&|Q_{n+1}(s,a,b)|\\&=|(1-\beta_n)Q_n(s,a,b)+\beta_n(r^{1}+\alpha val[Q_n(s^{1})]\\& \quad +\alpha \theta^1_n (r^2+\alpha val[Q_n(s^2)] )+\cdots\\&\quad \quad \cdots+\alpha^{k-1} \theta^{k-1}_n (r_1+\alpha val[Q_n(s_1)] ) ) |\\
		& \leq (1-\beta_n)\|Q_n\|+\beta_n\bigg(R_{\max}+\alpha \|Q_n\|+\alpha |\theta^{1}_n|(R_{\max}+\alpha \|Q_n\|)+\\
		& +\alpha^2 |\theta^{2}_n|(R_{\max}+\alpha \|Q_n\|)+...+\alpha^{k-1} |\theta^{k-1}_n|(R_{\max}+\alpha \|Q_n\|)\bigg)\\
		&\leq L (1+\alpha^2 \beta_n |\theta^{1}_n|+...+\alpha^{k}\beta_n |\theta^{k-1}_n|)\\
		&\quad -\beta_nL+\alpha\beta_nL+\beta_nR_{\max} (1+\alpha |\theta^{1}_n|+...+\alpha^{k-1} |\theta^{k-1}_n|).
	\end{align*}
	Now we substantiate that, $Z_n = -\beta_nL+\alpha\beta_nL+\beta_nR_{\max} (1+\alpha |\theta^{1}_n|+\cdots+\alpha^{k-1} |\theta^{k-1}_n|) \leq 0.$ For
	\begin{align*}
	&Z_n= \beta_n \left( (\alpha-1)L+R_{\max} (1+\alpha |\theta^{1}_n|+\cdots+\alpha^{k-1} |\theta^{k-1}_n|)\right)\\
	&= \beta_n \bigg( (\alpha-1)\frac{R_{\max}}{1-\alpha}(1+\alpha |\theta^{1}_0|+\alpha^2 |\theta^{2}_0|+\cdots\\
	&\cdots+\alpha^{k-1} |\theta^{k-1}_0|) \prod_{i=1}^{l-1} (1+\beta_i|\theta^{1}_i| \alpha^2+\beta_i|\theta^{2}_i| \alpha^3+\cdots\\
	&\cdots+\beta_i|\theta^{k-1}_i| \alpha^{k})+R_{\max} (1+\alpha |\theta^{1}_n|+\cdots+\alpha^{k-1} |\theta^{k-1}_n|)\bigg)\\
	&\leq \beta_n \bigg( -R_{\max}(1+\alpha |\theta^{1}_0|+\alpha^2 |\theta^{2}_0|+\cdots+\alpha^{k-1} |\theta^{k-1}_0|)+\\
	& \quad \quad R_{\max} (1+\alpha |\theta^{1}_n|+\cdots+\alpha^{k-1} |\theta^{k-1}_n|)\bigg)\leq 0.
	\end{align*}
	Therefore, $\left|Q_{n+1}(s,a,b)\right| \leq  L (1+\alpha^2 \beta_n |\theta^{1}_n|+\cdots+\alpha^{k}\beta_n |\theta^{k-1}_n|)$. Thus, $\max_{(s,a,b)} |Q_{n+1}(s,a,b)| = \|Q_{n+1}\| \leq\frac{R_{\max}}{1-\alpha}(1+\alpha |\theta^{1}_0|+\alpha^2 |\theta^{2}_0|+\cdots+\alpha^{k-1} |\theta^{k-1}_0|) \prod_{i=1}^{n} (1+\beta_i|\theta^{1}_i| \alpha^2+\beta_i|\theta^{2}_i| \alpha^3+\cdots+\beta_i|\theta^{k-1}_i| \alpha^{k})\leq M$. Hence, the result follows.
\end{proof}
Throughout this manuscript, we assume that the following property holds true.\\
\textbf{ Assumption:} The Markov chain generated by behaviour policy is ergodic. Furthermore, the behaviour policy has a non-zero probability of selecting all possible actions in any given state.	
\begin{theorem}
	Given a finite Markov game as defined in Section 2, the MMQL algorithm given by the update rule \eqref{main}, 
	converges $w.p.1$ to the optimal $Q^*$, where $\|Q_0\|\leq \frac{R_{\max}}{1-\alpha}$, and $0\leq \beta_n(i,a,b) \leq 1$ with $\sum_{n}\beta_n(i,a,b)=\infty,\; \sum_{n}\beta^2_n(i,a,b)<\infty,$ for all $(i,a,b)\in S\times A\times B$ with $\beta_n(i,a,b)=0, \forall (i,a,b) \neq (s_n,a_n,b_n)$.
\end{theorem}

\begin{proof}
	The correspondence to Lemma \ref{fl} follows from associating $X$ with $S\times A \times B$, and define $\Psi_n(s,a,b)=Q_n(s,a,b)-Q^*(s,a,b)$. Let

	\begin{equation*}
	\begin{split}
	s_n,a_n,b_n,r^{1}_n,s^{1}_n,a^{1}_n,b^{1}_n,r^{2}_n,s^{2}_n,a^{2}_n,b^{2}_n,r^{3}_n,s^{3}_n,a^{3}_n, b^{3}_n,\\
	r^{4}_n,\cdots, s^{k-1}_n,a^{k-1}_n,b^{k-1}_n,r_{n+1},s_{n+1},a_{n+1},b_{n+1}, \cdots
	\end{split}
	\end{equation*}
	be the trajectory produced by the behaviour strategies.	Using the above notation, the update rule can be written as
	\begin{align*} 
	&Q_{n+1}(s_n,a_n,b_n)\notag\\ &= (1-\beta_n(s_n,a_n,b_n))Q_n(s_n,a_n,b_n) \notag \\
	&\quad + \beta_n(s_n,a_n,b_n) \Big( r^{1}_n + \alpha\, val[Q_n(s^{1}_n)]+ \notag \\
	&\quad  \alpha\theta^{1}_n \left( r^{2}_n + \alpha\, val[Q_n(s^{2}_n)] \right)+ \alpha^2 \theta^{2}_n \left( r^{3}_n + \alpha\, val[Q_n(s^{3}_n)] \right)+ \notag \\
	&\quad \cdots + \alpha^{k-1} \theta^{k-1}_n \left( r_{n+1} + \alpha\, val[Q_n(s_{n+1})] \right) \Big),
	\end{align*}
	where $r^{1}_n=r(s_n,a_n,b_{n})$, $r^{1}_n=r(s^{1}_n,a^{1}_n,b^{1}_{n})$, $\cdots$, $r_{n+1}=r(s^{k-1}_{n},a^{k-1}_{n},b^{k-1}_{n})$.
	Thus,
	\begin{align*}
		&\Psi_{n+1}(s_n,a_n,b_n)\\&=(1-\beta_n(s_n,a_n,b_n))\Psi_{n}(s_n,a_n,b_n)+\beta_n(s_n,a_n,b_n)\\&\bigg(r^{1}_n+\alpha\, val[Q_n(s^{1}_{n})]+ \alpha \,\theta^{1}_n (r^{2}_n+ \alpha\, val[Q_n(s^{2}_{n})])+\cdots\\&\cdots+ \alpha^{k-1} \,\theta^{k-1}_n (r_{n+1}+\alpha \,val[Q_n(s_{n+1})] )-Q^*(s_n,a_n,b_n)\bigg).
	\end{align*}
	Let
	\begin{align*}
	&F_n(s_n,a_n,b_n)\\& = r^{1}_n+\alpha\, val[Q_n(s^{1}_{n})]+ \alpha \,\theta^{1}_n (r^{2}_n+ \alpha\, val[Q_n(s^{2}_{n})])+\cdots\\&\cdots+ \alpha^{k-1} \,\theta^{k-1}_n (r_{n+1}+\alpha \,val[Q_n(s_{n+1})] )-Q^*(s_n,a_n,b_n).
	\end{align*}
	As the multi-step minimax Q-learning requires extra sample during the update, for our analysis, we define the sigma fields as follows. For $n=0$, let $\mathcal{F}_0= \sigma(\{Q_0,s_0,a_0,b_0,\beta_0\})$ and for $n\geq 1$, $\mathcal{F}_n=\sigma(\{Q_0,s_0,a_0,b_0,\beta_0, r^{1}_{j-1},s^{1}_{j-1},a^{1}_{j-1},b^{1}_{j-1},\cdots, r^{k-1}_{j-1},s^{k-1}_{j-1},\\a^{k-1}_{j-1},b^{k-1}_{j-1}, r_j,\beta_j,s_j,a_j,b_j : 1 \leq j \leq n; \, k\geq 1 \})$.  
	With this choice of $\mathcal{F}_n$, $\beta_0$, and $\Psi_0$ are $\mathcal{F}_0$ measurable, and $\beta_n$, $\Psi_n$, and $F_{n-1}$ are $\mathcal{F}_n$ measurable.  Now we have,
		\begin{align*}
	&\big|E[F_n(s_n,a_n,b_n)|\mathcal{F}_n]\big|\\
		&=\big|E\big[r(s_n,a_n,b_n)+\alpha\, val[Q_n(s^{1}_n)]-Q^*(s_n,a_n,b_n)|\mathcal{F}_n\big]\\
		&\quad +\alpha \theta^{1}_n E\big[r^{2}_{n}+\alpha \,val[Q_n(s^{2}_{n})]|\mathcal{F}_n\big]\big|+\cdots\\
	&\quad\quad \quad \cdots+\alpha^{k-1} \theta^{k-1}_n E\big[r_{n+1}+\alpha \,val[Q_n(s_{n+1})]|\mathcal{F}_n\big]\big|\\
	&\leq\big|E\big[r(s_n,a_n,b_n)+\alpha\, val[Q_n(s'_n)]-Q^*(s_n,a_n,b_n)|\mathcal{F}_n\big]\big|\\&\quad +\alpha |\theta^{1}_n| \big|E\big[\big(r^{2}_n+\alpha \,val[Q_n(s_{n+1})] \big)|\mathcal{F}_n\big]\big|+\cdots\\
	&\quad \quad \cdots+ \alpha^{k-1} |\theta^{k-1}_n| \big|E\big[\big(r_{n+1}+\alpha \,val[Q_n(s_{n+1})] \big)|\mathcal{F}_n\big]\big|\\
	&=\bigg|\sum_{j=1}^{|S|}p(j|s_n,a_n,b_n)(r^{1}_n+\alpha val[Q_n(j)]-Q^*(s_n,a_n,b_n))\bigg|\\&\quad +\alpha |\theta^{1}_n| \big|E\big[\big(r^{2}_n+\alpha \,val[Q_n(s_{n+1})] \big)|\mathcal{F}_n\big]\big|+\cdots\\
	&\quad \quad \cdots+\alpha^{k-1} |\theta^{k-1}_n| \big|E\big[\big(r_{n+1}+\alpha \,val[Q_n(s_{n+1})] \big)|\mathcal{F}_n\big]\big|\\
	&\leq \left| HQ_n(s_n,a_n,b_n)-HQ^*(s_n,a_n,b_n)\right|\\&\quad+ (\alpha\, |\theta^{1}_n| + \cdots+  \alpha^{k-1}\, |\theta^{k-1}_n| )\left(R_{\max} + \alpha \|Q_n\|\right)\\
	& \leq \alpha \, \|\Psi_{n}\| + \zeta_n,
	\end{align*}

	\noindent
	where $\zeta_n=(\alpha\, |\theta^{1}_n| + \cdots+  \alpha^{k-1}\, |\theta^{k-1}_n| )\left(R_{\max} + \alpha \|Q_n\|\right)$. As the Q-iterates are bounded, and $|\theta^{i}_n|$ $\rightarrow$ 0, for $i=1,2,...,k-1$, one can conclude that, $\zeta_n$ $\rightarrow$ 0 as $n \rightarrow \infty$.
	Therefore, condition $(2)$ of Lemma \ref{fl} holds. \\
	Now consider,
	\begin{align*}	
	&Var[F_n(s_n,a_n,b_n)|\mathcal{F}_n]\\&=E\left[\left(F_n(s_n,a_n,b_n)-E[F_n(s_n,a_n,b_n)|\mathcal{F}_n]\right)^2 |\mathcal{F}_n\right]\\
	&\leq E\bigg[\bigg(r^{1}_n+\alpha val[Q_n(s^{1}_{n})]+\alpha \theta^{1}_n (r^{2}_n+\alpha val[Q_n(s^{2}_{n})])+...\\&\hspace*{1cm}\cdots+\alpha^{k-1} \theta^{k-1}_n (r_{n+1}+\alpha val[Q_n(s_{n+1})] )\bigg)^2|\mathcal{F}_n\bigg]\\
	& \leq \bigg(\big(1+\alpha |\theta^{1}_n|+\cdots+\alpha^{k-1} |\theta^{k-1}_n| )\big(R_{\max}+ \alpha \|Q_n\|\big)\bigg)^2\\
	& = Y^2(R_{\max}+ \alpha \|Q_n\|)^2, \\
	& \hspace*{1.6cm}\quad \text{where \,} Y=1+\alpha |\theta^{1}_n|+\cdots+\alpha^{k-1} |\theta^{k-1}_n|\\
	&\leq 2 Y^2 (R^2_{\max}+ \alpha^2 \|Q_n\|^2)\\
	&\leq 2 Y^2 \bigg(R^2_{\max}+ 2\alpha^2  (\|Q^{*}\|^2+\|\Psi_{n}\|^2)\bigg)\\
	& \leq K(1 + \|\Psi_n\|^2 ) \leq K(1 + \|\Psi_n\| )^2,
	\end{align*}
	where $K = \max\{2Y^2R_{\max}^2+ 4Y^2\alpha^2\|Q^*\| , \; 4Y^2\alpha\}$. All the hypotheses of Lemma \ref{fl} is fulfilled for the choice $X=S\times A\times B$, $\kappa = \alpha$. Hence, $\Psi_n$ $\rightarrow$ 0 $w.p.1$. Consequently, $Q_n$ converges to $Q^*$ $w.p.1.$
	
\end{proof}

	\section{Experiments}
\noindent
This section presents a comparison of the proposed multi-step minimax Q-learning (MMQL) algorithm with the classical minimax Q-learning (MQL) \cite{littman1994markov}, generalized optimal minimax Q-learning (G-SOROQL) \cite{diddigi2022generalized}, and generalized minimax Q-learning (G-SORQL) \cite{diddigi2022generalized}. The problem set-up is similar to that of the experiment performed in \cite{diddigi2022generalized}. It is essential to emphasize and acknowledge that the code utilized in this study has been sourced from the GitHub repository of R. B. Diddigi \cite{WinNT}. Our implementation can be accessed through the following link \cite{WinNT2}.
\subsection{Comparison of the MQL and MMQL}
\noindent
In this subsection TZMGs with 10 states and 5 actions for each player are generated. For various multi-step parameter $k$,  the performance of the proposed MMQL is compared with the classical MQL \cite{littman1994markov} in terms of error and execution time. Each episode of the experiment consists of running an algorithm for $1000$ iterations. The error upon completion of the episode is calculated by taking the Euclidean norm difference between the optimal value and the estimated value upon completion of the episode. The average error is evaluated by taking the average over $50$ independent episodes as follows
\begin{equation} \label{error1}
\text{Average Error} =\frac{1}{50}\sum_{m=1}^{50}\|Y^*-val[Q_{m}(.)]\|_2,
\end{equation}
where $Y^*$ is the min-max value function of the game and $val[Q_{m}(.)]$ is the min-max Q-value  estimated upon completion of the $m^{th}$ episode using the particular algorithm. Similarly, the average time (in seconds) is calculated by taking the average over the time taken by the algorithm to complete an episode.
\begin{table*}[ht]
	\centering
	\resizebox{0.8\textwidth}{!}{
		\begin{tabular}{|c|cc|cc|cc|cc|cc|}
			\hline
			\multirow{2}{*}{\textbf{Algorithm}} & \multicolumn{2}{c|}{\textbf{$\alpha$ = 0.1}}                                                                                                            & \multicolumn{2}{c|}{\textbf{$\alpha$ = 0.3}}                                                                                                             & \multicolumn{2}{c|}{\textbf{$\alpha$ = 0.5}}                                                                                                            & \multicolumn{2}{c|}{\textbf{$\alpha$ = 0.6}}                                                                                                             & \multicolumn{2}{c|}{\textbf{$\alpha$ = 0.7}}                                                                                                                     \\ \cline{2-11} 
			& \multicolumn{1}{c|}{\textbf{\begin{tabular}[c]{@{}c@{}}Average\\ Error\end{tabular}}} & \textbf{\begin{tabular}[c]{@{}c@{}}Average\\ Time\end{tabular}} & \multicolumn{1}{c|}{\textbf{\begin{tabular}[c]{@{}c@{}}Average\\ Error\end{tabular}}} & \textbf{\begin{tabular}[c]{@{}c@{}}Average \\ Time\end{tabular}} & \multicolumn{1}{c|}{\textbf{\begin{tabular}[c]{@{}c@{}}Average\\ Error\end{tabular}}} & \textbf{\begin{tabular}[c]{@{}c@{}}Average\\ Time\end{tabular}} & \multicolumn{1}{c|}{\textbf{\begin{tabular}[c]{@{}c@{}}Average\\ Error\end{tabular}}} & \textbf{\begin{tabular}[c]{@{}c@{}}Average \\ Time\end{tabular}} & \multicolumn{1}{c|}{\textbf{\begin{tabular}[c]{@{}c@{}}Average\\ Error\end{tabular}}} & \textbf{\begin{tabular}[c]{@{}c@{}}Average\\ Time\end{tabular}} \\ \hline
			\textbf{$k=1$}                        & \multicolumn{1}{c|}{0.044}                                                            & 1.118                                                           & \multicolumn{1}{c|}{0.0856}                                                           & 1.153                                                            & \multicolumn{1}{c|}{0.328}                                                            & 1.113                                                           & \multicolumn{1}{c|}{0.684}                                                            & 1.114                                                            & \multicolumn{1}{c|}{1.475}                                                            & 1.144                                                           \\ \hline
			\textbf{$k=2$}                        & \multicolumn{1}{c|}{0.083}                                                            & 1.690                                                           & \multicolumn{1}{c|}{0.132}                                                            & 1.730                                                            & \multicolumn{1}{c|}{0.184}                                                            & 1.695                                                           & \multicolumn{1}{c|}{0.181}                                                            & 1.712                                                            & \multicolumn{1}{c|}{0.300}                                                            & 1.709                                                           \\ \hline
			\textbf{$k=3$}                        & \multicolumn{1}{c|}{0.072}                                                            & 2.262                                                           & \multicolumn{1}{c|}{0.128}                                                            & 2.303                                                            & \multicolumn{1}{c|}{0.188}                                                            & 2.266                                                           & \multicolumn{1}{c|}{0.171}                                                            & 2.256                                                            & \multicolumn{1}{c|}{0.264}                                                            & 2.258                                                           \\ \hline
			\textbf{$k=4$}                        & \multicolumn{1}{c|}{0.065}                                                            & 2.771                                                           & \multicolumn{1}{c|}{0.122}                                                            & 2.877                                                            & \multicolumn{1}{c|}{0.175}                                                            & 2.845                                                           & \multicolumn{1}{c|}{0.167}                                                            & 2.845                                                            & \multicolumn{1}{c|}{0.268}                                                            & 2.823                                                           \\ \hline
		\end{tabular}%
	}
	\caption{The performance of MMQL for various choices of multi-step parameter $k$ and various discount factor $\alpha$.}
	\label{tab:my-table}
\end{table*}\\
In Table \ref{tab:my-table}, the performance of proposed algorithm MMQL with various choices of multi-step parameter $k$ is presented. The choices for $\theta^{1}_n$, $\theta^{2}_n$, and $\theta^{3}_n$ are $\frac{80}{n+80}$, $\frac{10}{n^2+10}$, $\frac{10}{n^2+10}$ respectively. Note that $k=1$ corresponds to the MQL algorithm, from Table \ref{tab:my-table}, one can see that MQL performs better for lower discount factors. Specifically, for $\alpha=0.1$, and $\alpha=0.3$. However, as $\alpha$ increases multi-step methods performs better than single-step method MQL in terms of accuracy. Figure \ref{nnn2}, provides the graphical representation of the Table \ref{tab:my-table}. From Table \ref{tab:my-table}, it is evident that the MMQL algorithm with $k=4$ performs better than $k=3$, $k=2$, and $k=1$ in terms of average error for higher discount factor. Note that though better accuracy is obtained as $k$ increases but the execution time is also increasing. Rest of the subsections the performance of the proposed multi-step method demonstrated  for the multi-step parameter $k=2$. 
\begin{figure}[ht]
	\includegraphics[width=9cm]{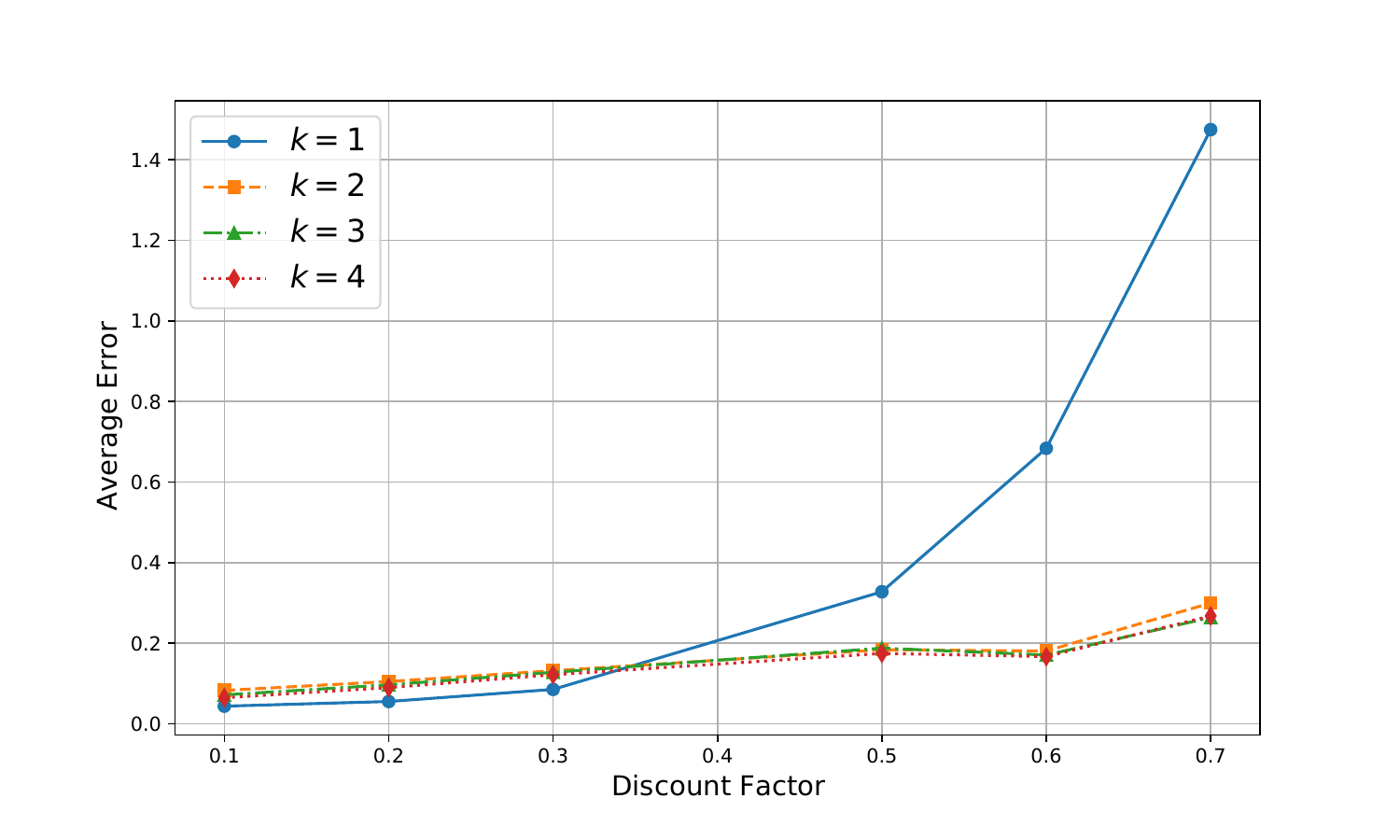}
	\centering
	\caption{ For $|S| =10$, $|A|=|B|=5$, discount factor vs average error.}
	\label{nnn2}
\end{figure}

\subsection{Experiment with hundred random MGs}
\noindent
In this subsection, we conduct experiments by generating $100$ MGs. For each fixed state-action pairs, $100$ Markov games generated and solved using the proposed algorithm with $k=2$ and the classical minimax Q-learning (MQL) \cite{littman1994markov}. For performance comparison the average error is calculated as follows:
\begin{equation}\label{eq2}
\text{Average Error} = \frac{1}{100}\sum_{m=1}^{100} \|J_m^* - val[Q_m(.)]\|_2,
\end{equation}
where $J_m^*$ is the min-max value function of  $m^{th}$ game, and $val[Q_m(.)]$ is the min-max Q-value corresponding to $m^{th}$ game obtained after running the particular algorithm for $1000$ iterations. Throughout this subsection the choice of $\beta_n = 10^2/(n+10^2)$, and $\theta^{1}_n = 10^3/(n+10^3)$. \\
In Table \ref{tab73}, we compare the proposed algorithm MMQL with multi-step parameter $k=2$ and MQL for various choices of discount factor. Table \ref{tab73} provides the number of iterations, and time required for MMQL with $k=2$ to achieve similar average error obtained by MQL after 1000 iterations. Specifically, for discount factor, 0.4, 0.5, 0.6, 0.7, and 0.8 on hundred randomly generated state-action pairs with dimension of $|S| =10$,  $|A|=4,$ and $ |B|=6$.  

\begin{table*}[ht]
	\centering
	\centering
	\resizebox{0.8\textwidth}{!}{
		\begin{tabular}{|c|cc|cc|cc|cc|cc|}
			\hline
			\multirow{2}{*}{\textbf{$\alpha$}}                                        & \multicolumn{2}{c|}{\textbf{0.4}}                   & \multicolumn{2}{c|}{\textbf{0.5}}                   & \multicolumn{2}{c|}{\textbf{0.6}}                   & \multicolumn{2}{c|}{\textbf{0.7}}                   & \multicolumn{2}{c|}{\textbf{0.8}}                   \\ \cline{2-11} 
			& \multicolumn{1}{c|}{\textbf{MMQL}}  & \textbf{MQL}  & \multicolumn{1}{c|}{\textbf{MMQL}}  & \textbf{MQL}  & \multicolumn{1}{c|}{\textbf{MMQL}}  & \textbf{MQL}  & \multicolumn{1}{c|}{\textbf{MMQL}}  & \textbf{MQL}  & \multicolumn{1}{c|}{\textbf{MMQL}}  & \textbf{MQL}  \\ \hline
			\textbf{Iteration}                                                & \multicolumn{1}{c|}{\textbf{215}}   & \textbf{1000} & \multicolumn{1}{c|}{\textbf{180}}   & \textbf{1000} & \multicolumn{1}{c|}{\textbf{160}}   & \textbf{1000} & \multicolumn{1}{c|}{\textbf{160}}   & \textbf{1000} & \multicolumn{1}{c|}{\textbf{160}}   & \textbf{1000} \\ \hline
			\textbf{\begin{tabular}[c]{@{}c@{}}Average \\ Error\end{tabular}} & \multicolumn{1}{c|}{0.502}          & 0.518         & \multicolumn{1}{c|}{0.804}          & 0.821         & \multicolumn{1}{c|}{1.338}          & 1.351         & \multicolumn{1}{c|}{2.310}          & 2.331         & \multicolumn{1}{c|}{4.430}          & 4.442         \\ \hline
			\textbf{\begin{tabular}[c]{@{}c@{}}Average \\ Time\end{tabular}}  & \multicolumn{1}{c|}{\textbf{0.248}} & 0.554         & \multicolumn{1}{c|}{\textbf{0.198}} & 0.554         & \multicolumn{1}{c|}{\textbf{0.176}} & 0.552         & \multicolumn{1}{c|}{\textbf{0.179}} & 0.556         & \multicolumn{1}{c|}{\textbf{0.183}} & 0.552         \\ \hline
		\end{tabular}
	}
	\caption{MMQL with $k=2$ vs MQL with $|S|=10$, $|A|=4$, and $|B|=6$.}
	\label{tab73}
\end{table*}
\noindent
From Table \ref{tab73}, it is evident that the proposed MMQL with $k=2$ achieves the target in fewer iterations and less computational time. Specifically, for discount factor, 0.4, 0.5, 0.6, 0.7, and 0.8, the MMQL algorithm achieves similar error as that of MQL in 215, 180, 160 iterations for the last three values, respectively. It is worth to mention that for the above choices of $\beta_n$, and $\theta^1_n$ the proposed multi-step method requires almost fifty percent less sample than single-step method MQL to produce similar accuracy.
\begin{table*}[ht]
	\centering
	\resizebox{0.8\textwidth}{!}{
		\begin{tabular}{|c|cc|cc|cc|cc|}
			\hline
			& \multicolumn{2}{c|}{\textbf{(10, 5, 5)}}                                                                                                                & \multicolumn{2}{c|}{\textbf{(10, 4, 6)}}                                                                                                                & \multicolumn{2}{c|}{\textbf{(10, 4, 7)}}                                                                                                                & \multicolumn{2}{c|}{\textbf{(10, 8, 4)}}                                                                                                                \\ \cline{2-9} 
			\multirow{-2}{*}{\textbf{Algorithm}} & \multicolumn{1}{c|}{\textbf{\begin{tabular}[c]{@{}c@{}}Average\\ Error\end{tabular}}} & \textbf{\begin{tabular}[c]{@{}c@{}}Average\\ Time\end{tabular}} & \multicolumn{1}{c|}{\textbf{\begin{tabular}[c]{@{}c@{}}Average\\ Error\end{tabular}}} & \textbf{\begin{tabular}[c]{@{}c@{}}Average\\ Time\end{tabular}} & \multicolumn{1}{c|}{\textbf{\begin{tabular}[c]{@{}c@{}}Average\\ Error\end{tabular}}} & \textbf{\begin{tabular}[c]{@{}c@{}}Average\\ Time\end{tabular}} & \multicolumn{1}{c|}{\textbf{\begin{tabular}[c]{@{}c@{}}Average\\ Error\end{tabular}}} & \textbf{\begin{tabular}[c]{@{}c@{}}Average\\ Time\end{tabular}} \\ \hline
			\textbf{MQL \cite{littman1994markov}}                         & \multicolumn{1}{c|}{1.667}                                    & 0.556                                   & \multicolumn{1}{c|}{1.351}                                    & 0.552                                   & \multicolumn{1}{c|}{1.333}                                    & 0.556                                   & \multicolumn{1}{c|}{2.211}                                    & 0.55                                    \\ \hline
			\textbf{MMQL}                        & \multicolumn{1}{c|}{\textbf{0.591}}                                    & 1.101                                   & \multicolumn{1}{c|}{\textbf{0.458}}                                    & 1.096                                   & \multicolumn{1}{c|}{\textbf{0.533}}                                     & 1.10                                    & \multicolumn{1}{c|}{\textbf{1.128}}                                    & 1.1                                     \\ \hline
		\end{tabular}%
	}
	\caption{MMQL with $k=2$ vs MQL with $|S|=10$ and various action pairs}
	\label{tab71}
\end{table*}\\
\noindent
Table \ref{tab71} provides a comparison of the MQL algorithm with the proposed MMQL for $k=2$ on a fixed state space of 10, with various action pairs for each player. In other words, Table \ref{tab71} is generated using the configurations $(10, 5, 5)$, $(10, 4, 6)$, $(10, 8, 4)$, and $(10, 4, 7)$. The average error is evaluated as in \eqref{eq2}. From Table \ref{tab71}, it is evident that the proposed algorithm achieves lower error compared to MQL in all the situations.

\subsection{Comparison with the recent algorithms:}
In this subsection, we compare the proposed algorithm with classical minimax Q-learning (MQL) \cite{littman1994markov}, together with the recently developed generalized minimax Q-learning algorithm (G-SORQL) \cite{diddigi2022generalized}, and generalized minimax optimal Q-learning algorithm (G-SOROQL)  \cite{diddigi2022generalized}. Note that G-SORQL and G-SOROQL are developed only for TZMG with the restriction : $p(i|i,a,b) > 0, \forall i,a,b$. Consequently, we impose this condition on the transition probability of the generated MGs. It is interesting to note that the G-SOROQL algorithm is not a model-free algorithm as mentioned in \cite{diddigi2022generalized}. In this subsection TZMGs with 10, 20, and 50 states and 5 actions for each player are generated. The discount factor $\alpha$ is set to 0.6.  The average error is obtained using \eqref{error1}. The choice of the step-size sequence is the same for all the experiments as in \cite{diddigi2022generalized}. The choice $\theta^{1}_n$ for this experiment is $\frac{80}{n+80}$. 
\begin{table}[ht]\centering
	\resizebox{8.5cm}{!}{
		\begin{tabular}{|c|cc|cc|cc|}
			\hline
			\textbf{}                                                                                 & \multicolumn{2}{c|}{\textbf{10 states}}                                                                                                                  & \multicolumn{2}{c|}{\textbf{20 states}}                                                                                                                   & \multicolumn{2}{c|}{\textbf{50 states}}                                                                                                                   \\ \hline
			\textbf{Algorithm}                                                                        & \multicolumn{1}{c|}{\textbf{\begin{tabular}[c]{@{}c@{}}Average\\ Error\end{tabular}}} & \textbf{\begin{tabular}[c]{@{}c@{}}Average \\ Time\end{tabular}} & \multicolumn{1}{c|}{\textbf{\begin{tabular}[c]{@{}c@{}}Average \\ Error\end{tabular}}} & \textbf{\begin{tabular}[c]{@{}c@{}}Average \\ Time\end{tabular}} & \multicolumn{1}{c|}{\textbf{\begin{tabular}[c]{@{}c@{}}Average \\ Error\end{tabular}}} & \textbf{\begin{tabular}[c]{@{}c@{}}Average \\ Time\end{tabular}} \\ \hline
			\textbf{\begin{tabular}[c]{@{}c@{}}MQL \cite{littman1994markov}\end{tabular}}            & \multicolumn{1}{c|}{0.6819}                                                           & 1.224                                                            & \multicolumn{1}{c|}{1.670}                                                             & 1.203                                                            & \multicolumn{1}{c|}{3.9927}                                                            & 1.244                                                            \\ \hline
			\textbf{\begin{tabular}[c]{@{}c@{}}G-SORQL \cite{diddigi2022generalized}\end{tabular}}         & \multicolumn{1}{c|}{0.5021}                                                           & 3.039                                                            & \multicolumn{1}{c|}{1.4550}                                                            & 4.340                                                            & \multicolumn{1}{c|}{3.7676}                                                            & 7.534                                                           \\ \hline
			\textbf{\begin{tabular}[c]{@{}c@{}}G-SOROQL \cite{diddigi2022generalized}\end{tabular}} & \multicolumn{1}{c|}{0.3364}                                                           & 1.263                                                            & \multicolumn{1}{c|}{1.2858}                                                            & 1.239                                                            & \multicolumn{1}{c|}{3.597}                                                             & 1.266                                                            \\ \hline
			\textbf{\begin{tabular}[c]{@{}c@{}}MMQL ($k=2$)\end{tabular}}            & \multicolumn{1}{c|}{\textbf{0.222}}                                                  & 1.841                                                            & \multicolumn{1}{c|}{\textbf{1.054}}                                                    & 1.789                                                            & \multicolumn{1}{c|}{\textbf{3.476}}                                                    & 1.897                                                            \\ \hline
		\end{tabular}%
	}
	\caption{Comparison of algorithms with varying states for $\alpha=0.6$, and $|A|=|B|=5$.}
	\label{tab2}
\end{table}\\
\noindent
Table \ref{tab2} provides the comparison of all the algorithms for different states. In all the cases, the proposed method shows superior performance in terms of average error. Note that the G-SORQL algorithm involves evaluating the successive relaxation term at every iteration of the algorithm \cite{diddigi2022generalized}. This increases the computational cost of the G-SORQL as the evaluation of successive relaxation term depends on the cardinality of the set $S$. This is evident in the time G-SORQL requires to finish an episode as the cardinality of state $S$ increases. On the other hand, the proposed MMQL with $k=2$ and existing MQL do not show much changes in the execution time as the state size increases. \\
 Figure \ref{nnn} depicts the convergence behaviour of all the algorithms. The figure clearly demonstrates that the average error is decreasing for all the algorithms. Moreover, it is evident that the proposed algorithm achieves faster convergence. 
\begin{figure}[ht]
	\includegraphics[width=9cm]{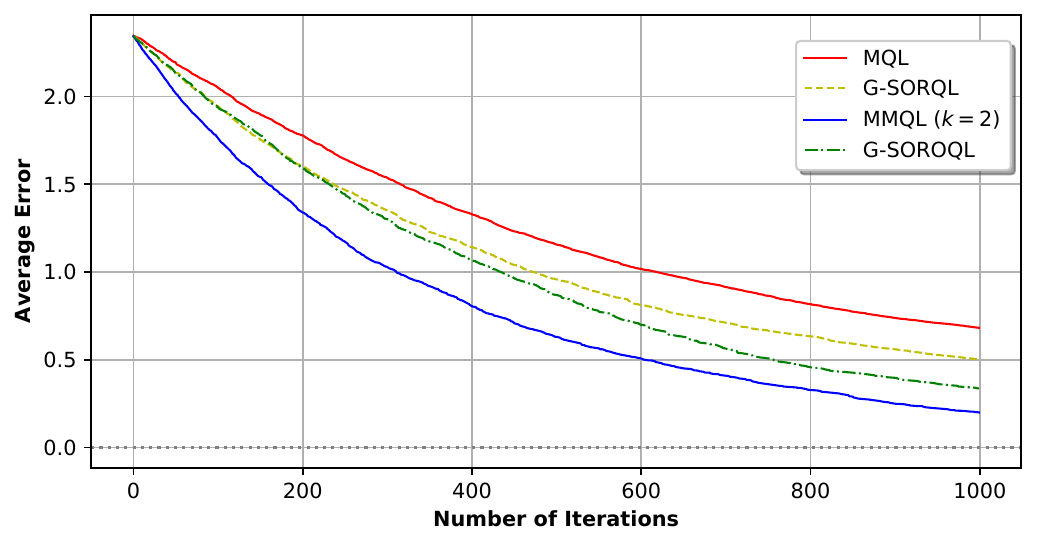}
	\centering
	\caption{ For $|S| =10$, $|A|=|B|=5$, iterations vs average error.}
	\label{nnn}
\end{figure}	
\section{Conclusion}
\noindent
This manuscript presents a novel multi-step RL algorithm for the two-player zero-sum Markov game. The boundedness of the multi-step minimax Q-learning and the convergence property of the algorithm is obtained theoretically. Finally, empirical tests are conducted to confirm the advantages of the proposed algorithm.

\bibliographystyle{IEEEtran}

\bibliography{ref}

\end{document}